\documentclass[a4paper]{article} 

\usepackage[hyphens]{url}  
\usepackage{graphicx} 

\usepackage[margin=3.5cm]{geometry}

\title{Truth Set Algebra: A New Way to Prove Undefinability}

\newcommand{\F}{{\sf F}}

\newcommand{\U}{{\sf U}}
\newcommand{\X}{{\sf X}}
\newcommand{\W}{{\sf W}}
\newcommand{\limp}{\to_{\text{\tiny \L}}}
\newcommand{\kimp}{\to_{\text{\tiny K}}}
\renewcommand{\S}{{\sf S}}
\renewcommand{\phi}{\varphi}
\usepackage{stmaryrd}
\usepackage{color}

  \usepackage{booktabs}
  \usepackage{amsmath,amssymb}
\renewcommand{\[}{\llbracket}
\renewcommand{\]}{\rrbracket}

\newtheorem{definition}{Definition}
\newtheorem{lemma}{Lemma}

\newtheorem{theorem}{Theorem}

\newenvironment{proof}{\noindent{\sc Proof.}}{\qed}
\newcommand{\qed}{\hfill $\boxtimes\hspace{2mm}$}

\newenvironment{proof-of-claim}{\noindent{\em Proof of Claim.}}{\qed}

\author{Sophia Knight\\
Department of Computer Science\\
University of Minnesota Duluth\\ 
the United States\\
\\
Pavel Naumov, Qi Shi, and Vigasan Suntharraj\\
Electronics and Computer Science\\
University of Southampton\\
the United Kingdom
}

\begin{document}

\maketitle

\begin{abstract}
The article proposes a new technique for proving the undefinability of logical connectives through each other and illustrates the technique with several examples. Some of the obtained results are new proofs of the existing theorems, others are original to this work.
\end{abstract}

\section{Introduction}

\nocite{m98}

Studying the definability (expressibility) of logical connectives in terms of one another has a long history in logic. Proving the definability of one connective through another is usually done by providing an explicit formula that expresses one connective through others. Once such a formula is found, proving definability is usually a straightforward exercise.  Proving {\em undefinability} is significantly harder and usually requires sophisticated techniques. Different domain-specific techniques have been proposed for various logical systems. Among them, the best-known is the {\em bisimulation} method for modal logics~\cite{vvs06ups,bc18sl,bv07handbook,f22jpl,no23synthese,ny21jair,dn21tocl,ny21aaai,nt21ijcai}. It is not clear how bisimulation can be applied to non-modal logics where completely different methods have been proposed~\cite{m39jsl,w38wm}. In addition, even for modal logics, some proofs of undefinability use non-bisimulation methods~\cite{m02entcs,l95ipl}.

In this article, we propose a new technique for proving the undefinability of logical connectives which is applicable to a wide range of settings. The technique consists in defining the ``truth set'' of a formula and studying the patterns of these truth sets obtainable through the given connectives. The exact definition of ``truth set'' varies depending on the logical system. For example, in the context of definability of Boolean connectives through each other, the truth set is defined as a set of {\em valuations} that satisfy a given formula. In the context of modal logics, the truth set is the set of {\em worlds} of a fixed given Kripke model in which the formula is true. In the context of three-valued logics, the ``truth set'' is a {\em fuzzy} set of valuations.

We illustrate this technique on the examples from Boolean, three-valued, intuitionistic, and temporal logics. We have chosen these specific examples to make the presentation accessible to a broader logical audience: we assume that most logicians are familiar with these logical systems. 

We use the Boolean logic example to introduce the basic idea behind our technique. We are not aware of any published work containing the undefinability result in that example, but it is so simple that we assume that somebody has observed it before. Our temporal logic and intuitionistic logic examples reprove known results using the newly proposed technique. We discuss the related literature after we present these results. Our 3-valued logic results are original to this article.

\section{Classical Propositional Logic}\label{Classical Propositional Logic section}




This section illustrates our technique using a simple undefinability result in propositional logic. In the rest of the article, we assume a fixed nonempty set of propositional variables. Consider language $\Phi_1$  defined by the following grammar:
$$
\phi := p\;|\;\neg\phi\;|\;\phi\wedge\phi\;|\;\phi\vee\phi\;|\;\phi\to\phi,
$$
where $p$ is a propositional variable.
As usual, we assume that constant $\top$ is defined as $p\to p$ for some propositional variable $p$ and constant $\bot$ is defined as $\neg\top$.
There are many well-known definability results in propositional logic:
\begin{align*}
    \phi\wedge \psi &\equiv \neg(\neg \phi\vee\neg \psi),\\
    \phi\wedge \psi &\equiv \neg(\phi\to\neg \psi),\\
    \phi\vee \psi &\equiv \neg(\neg \phi\wedge\neg \psi),\\
    \phi\vee \psi &\equiv \neg\phi\to\psi,\\
    \phi\to\psi &\equiv \neg\phi\vee\psi,\\
    \phi\to\psi &\equiv \neg(\phi\wedge\neg\psi).
\end{align*}
However, it is perhaps less known that disjunction can be defined through implication alone without the negation:
$$\phi\vee\psi\equiv (\phi\to\psi)\to\psi.$$
The last fact and the well-known symmetry between disjunction and conjunction in propositional logic naturally lead to the question of whether conjunction can be defined solely through implication. Perhaps surprisingly, the answer is negative and we prove this as our first example. 

Before formally stating the result, we introduce several auxiliary notions. First, a valuation is an arbitrary assignment of Boolean values to propositional variables. Second, for any formula $\phi\in\Phi_1$, by $\[\phi\]$ we denote the set of all valuations that satisfy formula $\phi$. We refer to set $\[\phi\]$ as the ``truth set'' of formula $\phi\in\Phi_1$. Finally, we define the semantic equivalence of formulae:
\begin{definition}\label{sem eq Boolean}
    Propositional formulae $\phi,\psi\in\Phi_1$ are semantically equivalent if $\[\phi\]=\[\psi\]$. 
\end{definition}
Next is our first undefinability result. 
\begin{theorem}[undefinability]\label{1-may-d}
The formula $p\wedge q$ is not semantically equivalent to any formula in language $\Phi_1$ containing only connectives $\vee$ and $\to$.
\end{theorem}
Because the formula $p\wedge q$ contains only propositional variables $p$ and $q$, without loss of generality, we can assume the language $\Phi_1$ contains only propositional variables $p$ and $q$. As a first step towards the proof, we introduce a way to visualise the truth set of any formula in language $\Phi_1$ using ``diagrams''. As an example, the diagram for the truth set $\[p\wedge q\]$ is depicted in Figure~\ref{conjunction figure}. 
\begin{figure}[ht]
\begin{center}
\vspace{0mm}
\scalebox{0.6}{\includegraphics{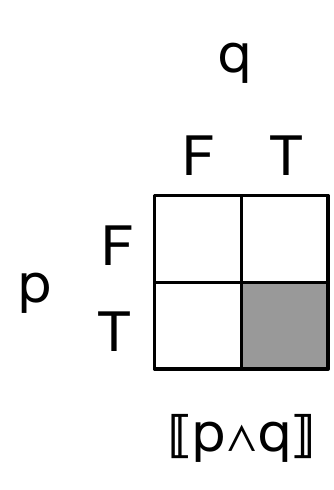}}
\caption{Truth set diagram.}\label{conjunction figure}
\end{center}
\end{figure}
In general, a diagram is a $2\times 2$ table whose cells represent valuations (mappings of the set $\{p,q\}$ into Boolean values). 
In the diagram, the cells representing elements of the given truth set are shaded grey. 
In other words, each element of the truth set of formula $\phi$ represents a valuation under which formula $\phi$ is true.
As another example, the diagrams at the top of Figure~\ref{intro figure} depict the truth sets $\[p\]$, $\[q\]$, $\[p\vee q\]$, $\[p\to q\]$, $\[q\to p\]$, and $\[\top\]$.

\begin{figure}[ht]
\begin{center}
\vspace{0mm}
\scalebox{0.53}{\includegraphics{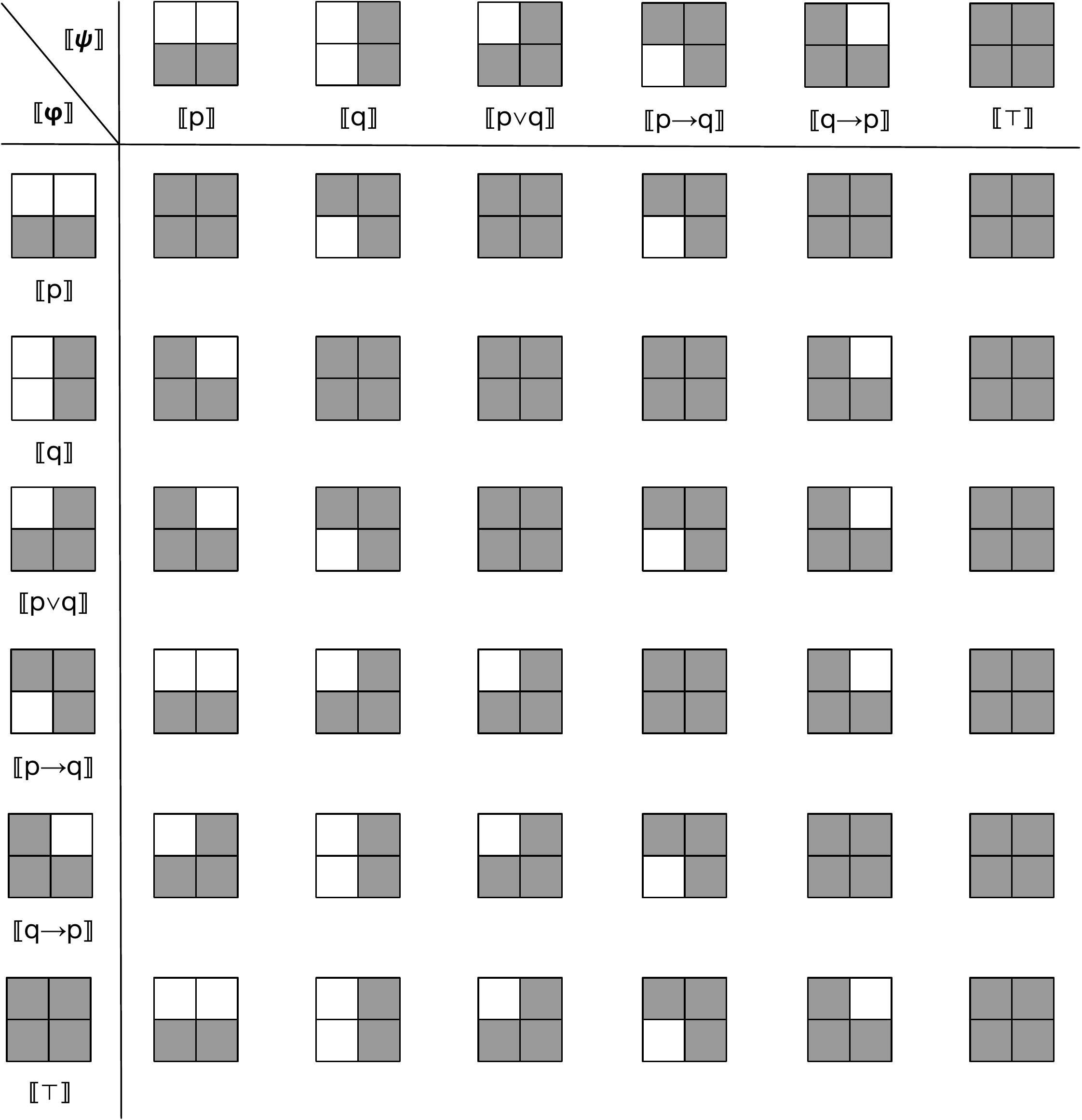}}
\caption{Truth set $\[\phi\to \psi\]$ for different combinations of truth sets $\[\phi\]$ and $\[\psi\]$.}\label{intro figure}
\end{center}
\end{figure}



The next lemma is the key step in our technique.
\begin{lemma}\label{1-may-a}
$\[\phi\to\psi\]\in \{\[p\],\[q\],\[p\vee q\], \[p\to q\], \[q\to p\], \[\top\]\}$ for any formulae $\phi,\psi\in\Phi_1$ such that $\[\phi\],\[\psi\]\in \{\[p\],\[q\],\[p\vee q\], \[p\to q\], \[q\to p\], \[\top\]\}$.
\end{lemma}
The lemma is proven by considering $6\times 6=36$ different cases corresponding to different combinations of possible values of sets $\[\phi\]$ and $\[\psi\]$. We show all these cases in Figure~\ref{intro figure}. For example, if $\[\phi\]=\[p\to q\]$ and  $\[\psi\]=\[q\]$, then $\[\phi\to\psi\]=\[p\vee q\]$. We show this in Figure~\ref{intro figure} by placing the diagram of the set  $\[p\vee q\]$ in the cell located at the intersection of the row labelled with the diagram $\[p\to q\]$ and the column labelled with the diagram $\[q\]$.

\begin{lemma}\label{1-may-b}
$\[\phi\]\in \{\[p\],\[q\],\[p\vee q\], \[p\to q\], \[q\to p\], \[\top\]\}$ for any formula $\phi\in\Phi_1$ that uses only connective $\to$.
\end{lemma}
\begin{proof}
The lemma is proven by induction on the structural complexity of formula $\phi$. The base case is true because truth sets $\[p\]$ and $\[q\]$ belong to the family of truth sets $\{\[p\],\[q\],\[p\vee q\], \[p\to q\], \[q\to p\], \[\top\]\}$. The induction step follows from Lemma~\ref{1-may-a}.   
\end{proof}

\begin{lemma}\label{1-may-c}
$\[p\wedge q\]\notin \{\[p\],\[q\],\[p\vee q\], \[p\to q\], \[q\to p\], \[\top\]\}$.
\end{lemma}
\begin{proof}
See Figure~\ref{conjunction figure} and the top row in Figure~\ref{intro figure}.    
\end{proof}
The statement of Theorem~\ref{1-may-d} follows from Lemma~\ref{1-may-b}, Lemma~\ref{1-may-c}, and Definition~\ref{sem eq Boolean}.

\section{Temporal Logic}

In this section, we show how the truth set algebra technique can be used to prove the undefinability of one modality through another. To do this, we use several modalities from linear temporal logic. We assume that time is discrete, starts at moment 0, and runs ad infinitum. We denote the set of nonnegative integers by $\mathbb{N}$. In the context of temporal logic, a valuation is any function $\pi$ that maps propositional variables into subsets of $\mathbb{N}$. 

The language $\Phi_2$ of temporal logic is defined by the following grammar:
$$
\phi := p\;|\;\neg\phi\;|\;\phi\vee\phi\;|\;\F\phi\;|\;\X\phi\;|\;\phi\U\phi\;|\;\phi\W\phi,
$$
where $p$ is either of the two propositional variables.
We read $\F$ as ``at some point in the future'', $\X$ as ``at the next moment'', $\U$ as ``until'', and $\W$ as ``weak until''. The formal semantics of these modalities is defined below.

\begin{definition}\label{sat temporal}
For any fixed valuation $\pi$, any integer $n\in\mathbb{N}$, and any formula $\phi\in\Phi_2$, the satisfaction relation $n\Vdash \phi$ is defined recursively as follows:
\begin{enumerate}
    \item $n\Vdash p$ if $n\in\pi(p)$,
    \item $n\Vdash \neg\phi$ if $n\nVdash\phi$,
    \item $n\Vdash \phi\vee\psi$ if either $n\Vdash\phi$ or $n\Vdash\psi$,
    \item $n\Vdash \F\phi$ if there is $m\ge n$ such that $m\Vdash\phi$,
    \item $n\Vdash\X\phi$ if $n+1\Vdash\phi$,
    \item $n\Vdash \phi\U\psi$ when there is $m\ge n$ such that $m\Vdash\psi$ and for each $i$, if $n\le i<m$, then $i\Vdash\phi$,
    \item $n\Vdash \phi\W\psi$, when for each $m\ge n$ such that $m\nVdash \phi$, there is $m'\ge n$ such that $m'\Vdash\psi$ and for each $i$, if $n\le i<m'$, then $i\Vdash\phi$.
\end{enumerate}
\end{definition}
Note that item~4 of the above definition contains inequality $m\le n$ rather than $m<n$. Thus, informally, in our system ``the future'' includes the current moment. We believe that this is a common approach in temporal logic, but this choice is not significant for our results.

\begin{definition}\label{truth set in temporal context}
In the context of temporal logic, 
for any given valuation $\pi$, let the truth set $\[\phi\]$ of a formula $\phi\in\Phi_2$  be the set $\{n\in\mathbb{N}\;|\;n\Vdash\phi\}$. 
\end{definition}

\begin{definition}\label{semantically equivalent in temporal context}
In the context of temporal logic, formulae $\phi,\psi\in\Phi_2$ are semantically equivalent if $\[\phi\]=\[\psi\]$ for each valuation $\pi$. 
\end{definition}

\subsection{Undefinability of $\U$ and $\W$ through $\F$}\label{U and W via F}

In this subsection, we use the truth set algebra method to show that both versions of ``until'' modalities, regular $\U$ and weak $\W$, are not definable through modality $\F$ and Boolean connectives. Without loss of generality, we assume that our language contains only propositional variables $p$ and $q$.
To start the proof, consider valuation $\pi$ defined as follows:
\begin{align*}
    \pi(p)=&\{n\ge 0\;|\; n\equiv 1\pmod 2\},\\
    \pi(q)=&\{n\ge 0\;|\; n\equiv 0\pmod 4\}.
\end{align*}
We visualise the truth sets of temporal formulae by drawing a one-way infinite linear sequence of cells and shading grey the cells whose position index belongs to the truth set (the left-most position corresponds to moment 0). The linear sequences in Figure~\ref{temporal figure} labelled with $\[p\]$ and $\[q\]$ visualise the corresponding truth sets. It is easy to verify that the other sequences also visualise the truth sets with which they are labelled.

\begin{figure}[ht]
\begin{center}
\vspace{0mm}
\scalebox{0.6}{\includegraphics{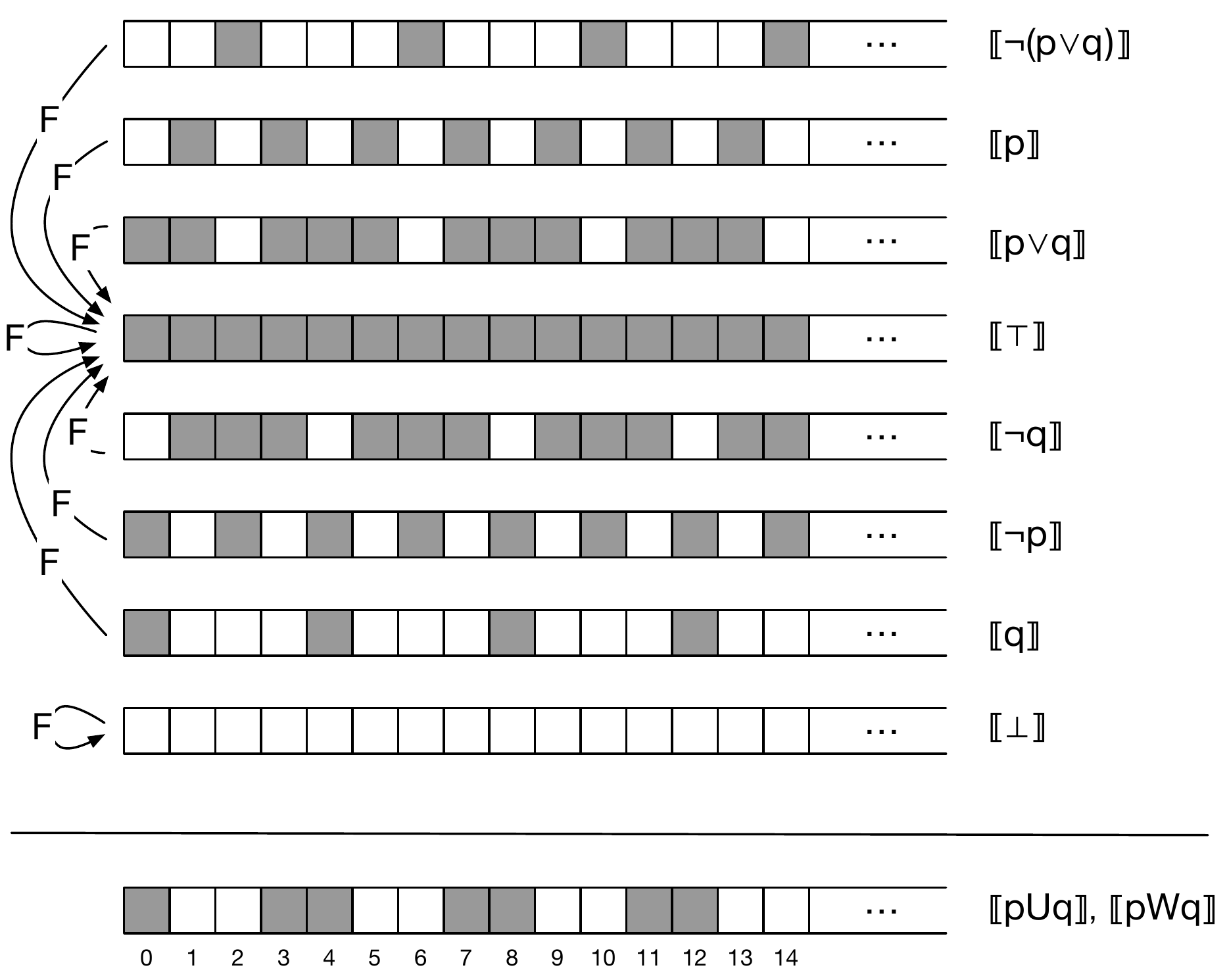}}
\caption{Visualisation of nine truth sets.}\label{temporal figure}
\end{center}
\end{figure}

The next lemma shows that the set of eight truth sets depicted {\em above the horizontal bar} in Figure~\ref{temporal figure} is closed with respect to modality $\F$.

\begin{lemma}\label{temporal induction step}
$\[\F\phi\]\in\{\[\top\],\[\bot\]\}$ for any temporal formula $\phi\in \Phi_2$ such that
$\[\phi\]\in\{\[\neg(p\vee q)\],\[p\],\[p\vee q\],\[\top\],\[\neg q\],\[\neg p\],\[q\],\[\bot\]\}$.
\end{lemma}
\begin{proof}
If $\[\phi\]\in\{\[\neg(p\vee q)\],\[p\],\[p\vee q\],\[\top\],\[\neg q\],\[\neg p\],\[q\]\}$, then statement $n\Vdash \phi$ holds for infinitely many values of $n$, see Figure~\ref{temporal figure}. Thus,  $n\Vdash \F\phi$  for each natural number $n$ by item~4 of Definition~\ref{sat temporal}. Therefore, $\[\F\phi\]=\[\top\]$ by Definition~\ref{truth set in temporal context}.

If $\[\phi\]=\[\bot\]$, then $n\nVdash\phi$ for each integer $n\ge 0$. Hence, $n\nVdash \F\phi$ for each $n$ by item~4 of Definition~\ref{sat temporal}. Therefore, $\[\F\phi\]=\[\bot\]$ by Definition~\ref{truth set in temporal context}.
\end{proof}

\begin{lemma}
$\[\phi\]\in \{\[\neg(p\vee q)\],\[p\],\[p\vee q\],\[\top\],\[\neg q\],\[\neg p\],\[q\],\[\bot\]\}$ for any temporal formula $\phi\in\Phi_2$ that does not contain modalities $\X$, $\U$, and $\W$.    
\end{lemma}
\begin{proof}
We prove the lemma by induction on the structural complexity of formula $\phi$. For the base case, note that 
$$\[p\],\[q\]\in \{\[\neg(p\vee q)\],\[p\],\[p\vee q\],\[\top\],\[\neg q\],\[\neg p\],\[q\],\[\bot\]\}.$$

Suppose formula $\phi$ has the form $\neg\psi$. By item~2 of Definition~\ref{sat temporal} and Definition~\ref{truth set in temporal context}, the truth set $\[\neg\psi\]$ is the complement of the truth set $\[\psi\]$. Note that the complement of each set in the family $\{\[\neg(p\vee q)\],\[p\],\[p\vee q\],\[\top\],\[\neg q\],\[\neg p\],\[q\],\[\bot\]\}$ also belongs to the same family. This can be observed in Figure~\ref{temporal figure}. For example, the complement of the set $\[\neg (p\vee q)\]$ is the set $\[p\vee q\]$. Therefore, set $\[\neg\psi\]$ belongs to the family $\{\[\neg(p\vee q)\],\[p\],\[p\vee q\],\[\top\],\[\neg q\],\[\neg p\],\[q\],\[\bot\]\}$ by the induction hypothesis. 

Assume that formula $\phi$ has the form $\psi_1\vee \psi_2$. By item~3 of Definition~\ref{sat temporal} and Definition~\ref{truth set in temporal context}, the truth set $\[\psi_1\vee \psi_2\]$ is the union of the truth sets $\[\psi_1\]$ and $\[\psi_2\]$. Note that the family $\{\[\neg(p\vee q)\],\[p\],\[p\vee q\],\[\top\],\[\neg q\],\[\neg p\],\[q\],\[\bot\]\}$ is closed with respect to union. This can also be observed in Figure~\ref{temporal figure}. For example, the union of the sets $\[\neg (p\vee q)\]$ and $\[p\]$ is the set $\[\neg q\]$. Therefore, set $\[\psi_1\vee \psi_2\]$ belongs to the family $\{\[\neg(p\vee q)\],\[p\],\[p\vee q\],\[\top\],\[\neg q\],\[\neg p\],\[q\],\[\bot\]\}$ by the induction hypothesis.  

If formula $\phi$ has the form $\F\psi$, then the statement of the lemma follows from Lemma~\ref{temporal induction step} and the induction hypothesis.
\end{proof}

\begin{lemma}
$\[p\U q\], \[p\W q\]\notin\{\[\neg(p\vee q)\],\[p\],\[p\vee q\],\[\top\],\[\neg q\],\[\neg p\],\[q\],\[\bot\]\}$.    
\end{lemma}
\begin{proof}
The truth sets $\[p\U q\]$ and $\[p\W q\]$ are equal and they are visualised below the horizontal bar in Figure~\ref{temporal figure}. The correctness of the visualisation follows from item~6 and item~7 of Definition~\ref{sat temporal}. Observe that these sets are different from the sets $\[\neg(p\vee q)\]$, $\[p\]$, $\[p\vee q\]$, $\[\top\]$, $\[\neg q\]$, $\[\neg p\]$, $\[q\]$, and $\[\bot\]$ visualised above the horizontal bar on the same diagram.
\end{proof}

The next result follows from Definition~\ref{semantically equivalent in temporal context} and the two lemmas above. A similar result for brunching time logic is shown in~\cite{m02entcs} using a different technique. Other undefinability results for a temporal logic are given in~\cite{l95ipl}.

\begin{theorem}[undefinability]
Neither the formula $p\U q$ nor the formula $p\W q$ is semantically equivalent to a formula in language $\Phi_2$ that does not contain modalities $\X$, $\U$, and $\W$.
\end{theorem}

\subsection{Undefinability of $\F$ through $\X$}

In this subsection, we use a modified version of the truth set algebra method to show that modality $\F$ is not definable through modality $\X$ and Boolean connectives. Without loss of generality, in this subsection, we assume that our language contains only propositional variable $p$.

In Subsection~\ref{U and W via F}, we have shown that a certain pattern can \textit{never} be reached by applying only modality $\F$ and Boolean connectives. Here we show that a certain pattern cannot be reached \textit{in a fixed number of steps} and use this observation to prove the undefinability.

We state and prove the undefinability result as Theorem~\ref{undefinability F via X theorem} at the end of this subsection. Throughout this subsection, until the statement of that theorem, we assume that $T\ge 1$ is an arbitrary fixed positive integer. We specify the value of $T$ in the proof of Theorem~\ref{undefinability F via X theorem}. Consider a valuations $\pi$ defined as follows:
\begin{equation} \label{13-03-2022 valuation}
   \pi(p)=\{T\}.  
\end{equation}
We visualise the truth sets in the same way as we did in the previous subsection. In Figure \ref{temporal figure infinite}, the top linear sequence visualises the truth set $\[p\]$.

For each integer $t$ such that $1\leq t\leq T$, we consider families of sets $\alpha_t$ and $\beta_t$ defined as
\begin{align*}
\alpha_t=&\big\{X\;|\;X\subseteq\{t,\dots,T\}\big\},\\  
\beta_t=&\big\{\{0,\dots,t-1\}\cup X\cup\{T+1,\dots\}\;|\;X\subseteq\{t,\dots,T\}\big\}. 
\end{align*}
In other words, $\alpha_t$ is the powerset of the set $\{t,\dots,T\}$ and $\beta_t$ is the set of the complements of sets in $\alpha_t$ with respect to $\mathbb{N}$. We visualise families $\alpha_t$ and $\beta_t$ in the middle of Figure \ref{temporal figure infinite}.  The asterisk $*$ is used as the ``wildcard'' to mark the integers that \textit{may but do not have to} belong to a set in the corresponding family. 
It is easily seen that for any integer $t\geq2$,
\begin{equation}\label{18-may-a}
    \alpha_{t}\subsetneq\alpha_{t-1} \;\;\;\;\mbox{ and }\;\;\;\ \beta_{t}\subsetneq\beta_{t-1}.
\end{equation}

\begin{figure}[ht]
\begin{center}
\vspace{0mm}
\scalebox{0.6}{\includegraphics{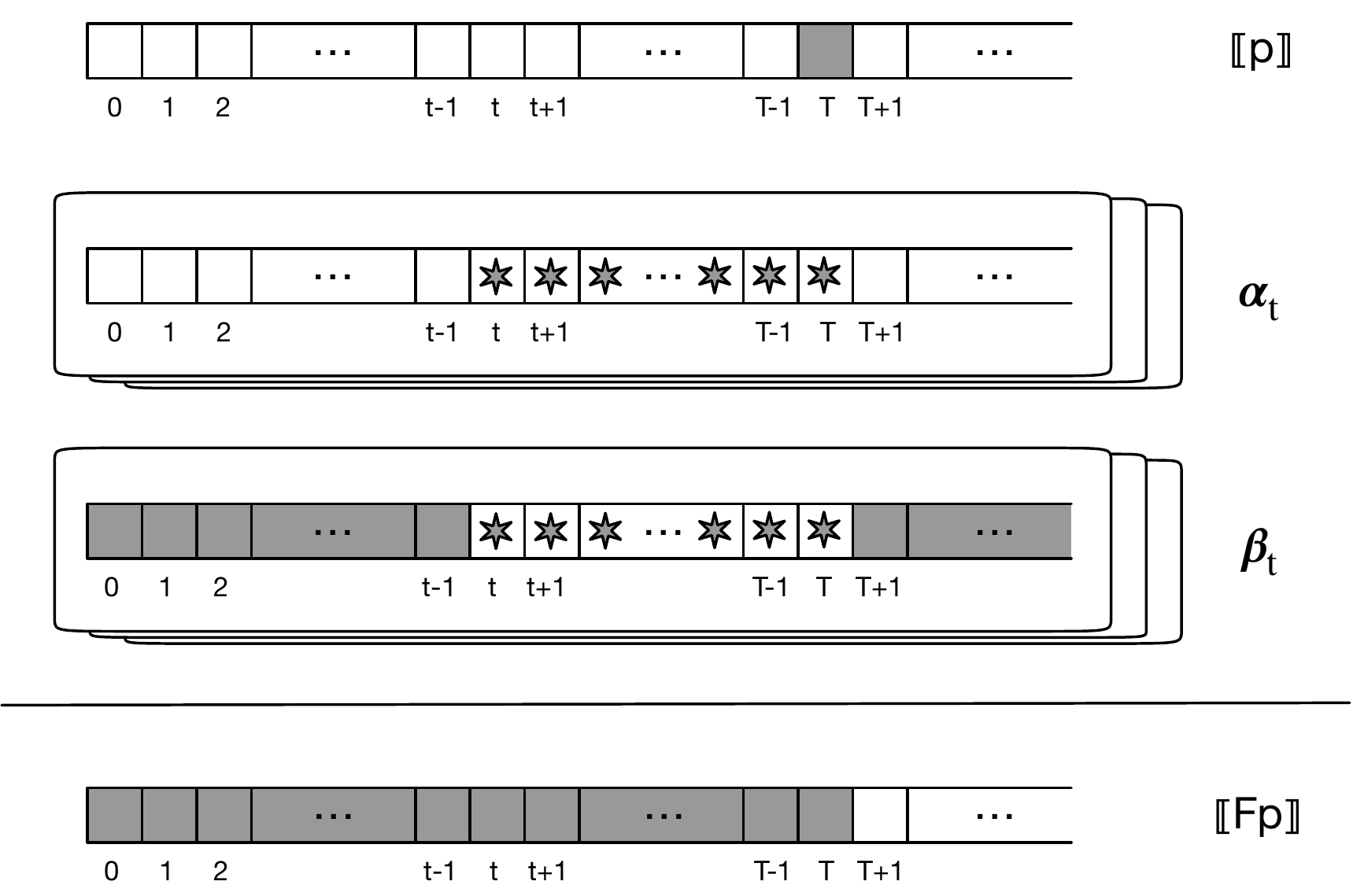}}
\caption{Visualisation of the truth sets for valuation $\pi^T$.}\label{temporal figure infinite}
\end{center}
\end{figure}

\begin{lemma} \label{lm:close disjunction negation}
For any formulae $\phi,\psi\in\Phi_2$ and any $t\geq1$, if $\[\phi\],\[\psi\]\in\alpha_{t}\cup\beta_{t}$, then $\[\phi\vee\psi\],\[\neg\phi\]\in\alpha_t\cup\beta_t$.
\end{lemma}
\begin{proof}
Observe from Figure~\ref{temporal figure infinite} that the family of sets $\alpha_t\cup\beta_t$ is closed with respect to union and complement. Then, the statement of the lemma follows from item~2 and item~3 of Definition~\ref{sat temporal}.
\end{proof}

\begin{lemma} \label{lm:close modality X}
For any formulae $\phi,\psi\in\Phi_2$ and any $t\geq1$, if $\[\phi\]\in\alpha_{t}\cup\beta_{t}$, then $\[\X\phi\]\in\alpha_{t-1}\cup\beta_{t-1}$.

\end{lemma}
\begin{proof}
By item 5 of Definition \ref{sat temporal}, $\[\X\phi\]=\{i\,|\,i\in\mathbb{N},i+1\in\[\phi\]\}$. Thus, if $\[\phi\]\in\alpha_t$, then $\[\X\phi\]\in\alpha_{t-1}$; if $\[\phi\]\in\beta_t$, then $\[\X\phi\]\in\beta_{t-1}$.
Therefore, when $\[\phi\]\in\alpha_{t}\cup\beta_{t}$, $\[\X\phi\]\in\alpha_{t-1}\cup\beta_{t-1}$.
\end{proof}

\begin{lemma} \label{lm:close form}
For any integer $k\le T$ and any formula $\phi\in\Phi_2$ that contains only modality $\X$ and Boolean connectives, if formula $\phi$ contains at most $k$ occurrences of modality $\X$, then $\[\phi\]\in\alpha_{T-k}\cup\beta_{T-k}$.
\end{lemma}
\begin{proof}
We prove the statement of this lemma by structural induction on formula $\phi$. If $\phi$ is a propositional variable $p$, then 
$$\[\phi\]=\[p\]=\pi(p)=\{T\}\in\alpha_{T}\subseteq\alpha_{T-k}$$
by item~1 of Definition~\ref{sat temporal}, statement~\eqref{13-03-2022 valuation}, and statement~\eqref{18-may-a}.

If formula $\phi$ is a disjunction or a negation, then the statement of this lemma follows from the induction hypothesis by Lemma \ref{lm:close disjunction negation}.

If formula $\phi$ has the form $\X \psi$, then formula $\psi$ contains at most $k-1$ occurrences of modality $\X$. The statement of this lemma follows from Lemma \ref{lm:close modality X}.
\end{proof}

\begin{lemma} \label{lm:modality F unreachable pattern}
If $T\geq1$, then $0\in\[\F p\]$ and $T+1\notin\[\F p\]$.
\end{lemma}
\begin{proof}
Since $\[p\]=\pi(p)=\{T\}$, by item 4 of Definition \ref{sat temporal}. Then, $\[\F p\]=\{0,\dots,T\}$, see the bottom linear sequence in Figure \ref{temporal figure infinite}. Therefore, $0\in\[\F p\]$ and $T+1\notin\[\F p\]$.
\end{proof}

The next theorem shows that modality $\F$ is not definable through modality $\X$ and Boolean connectives.

\begin{theorem}[undefinability]\label{undefinability F via X theorem}
The formula $\F p$ is not semantically equivalent to any formula in language $\Phi_2$ that does not contain modalities $\F$, $\U$, and $\W$.
\end{theorem}
\begin{proof}
Assume 
there is a formula $\phi\in\Phi_2$ that contains only modality $\X$ and Boolean connectives which is semantically equivalent to $\F p$. Suppose $k$ to be the number of occurrences of modality $\X$ in formula $\phi$. Let $T=k+1$.
Then, $\[\phi\]\in\alpha_{k+1-k}\cup\beta_{k+1-k}=\alpha_1\cup\beta_1$ by Lemma \ref{lm:close form}.
However, $\[\F p\]\notin\alpha_1\cup\beta_1$ by Lemma \ref{lm:modality F unreachable pattern}. Therefore, $\[\F p\]\neq\[\phi\]$, which contradicts   the assumption that formulae $\F p$ and $\phi$ are semantically equivalent by Definition~\ref{semantically equivalent in temporal context}.
\end{proof}

\section{Intuitionistic Logic}

In this section, we illustrate the truth set algebra method by proving the mutual undefinability of connectives in Heyting~\cite{h30} calculus for intuitionistic logic. These results\footnote{McKinsey~\cite{m39jsl} and Wajsberg~\cite{w38wm} talk about definability in terms of {\em provable} equivalence not {\em semantical} equivalence that we use in this article. The provable equivalence is equal to semantical equivalence due to the completeness theorem for intuitionistic logic proven by Kripke~\cite{k65slfm} in 1965.} were independently obtained by McKinsey~\cite{m39jsl} and Wajsberg~\cite{w38wm} in 1939.
Note that there were no Kripke semantics~\cite{k65slfm} for intuitionistic logic at the time~\cite{m39jsl,w38wm} were written. Our proof of definability uses Kripke models and, thus, is also significantly different from the original proofs in~\cite{m39jsl,w38wm}.

We start by recalling the standard Kripke semantics for intuitionistic logic~\cite{m22sep}. As usual, by ``partial order'' we mean a reflexive, transitive, and antisymmetric binary relation.

\begin{definition}\label{intuitionistic Kripke model}
An intuitionistic Kripke model is a tuple $(W,\preceq, \pi)$, where
\begin{enumerate}
    \item $W$ is a (possibly empty) set of ``worlds'',
    \item $\preceq$ is a partial order  on set $W$,
    \item for each propositional variable $p$, valuation  $\pi(p)\subseteq W$ is a set of worlds such that for any worlds $w,u\in W$, if $w\in \pi(p)$ and $w\preceq u$, then $u\in \pi(p)$.
\end{enumerate}
\end{definition}

In this section, we use the same language $\Phi_1$ as defined in Section~\ref{Classical Propositional Logic section}.

\begin{definition}\label{sat intuitionistic}
For any world $w\in W$ of a Kripke model $(W,\preceq, \pi)$ and any formula $\phi\in \Phi_1$, the satisfaction relation $w\Vdash\phi$ is defined as follows:
\begin{enumerate}
    \item $w\Vdash p$, if $w\in\pi(p)$,
    \item $w\Vdash \neg\phi$, if there is no world $u\in W$ such that $w\preceq u$ and $u\Vdash\phi$,
    \item $w\Vdash \phi\wedge \psi$, if $w\Vdash \phi$ and  $w\Vdash \psi$,
    \item $w\Vdash \phi\vee \psi$, if either $w\Vdash \phi$ or  $w\Vdash \psi$,
    \item $w\Vdash \phi\to \psi$, when for each world $u\in W$ if $w\preceq u$ and $u\Vdash \phi$, then $u\Vdash \psi$.
\end{enumerate}
\end{definition}
Note that item 3 of Definition~\ref{intuitionistic Kripke model} and items 2 and 5 of Definition~\ref{sat intuitionistic} capture the intuitionistic nature of this semantics.

\begin{definition}\label{truth set intuitionistic}
For any given intuitionistic  Kripke model $(W,\preceq, \pi)$, the truth set $\[\phi\]$ of an arbitrary formula $\phi\in\Phi_1$ is the set $\{w\in W\;|\;w\Vdash\phi\}$.    
\end{definition}

\begin{definition}
In the context of intuitionistic logic, formulae $\phi,\psi\in\Phi_1$ are semantically equivalent if $\[\phi\]=\[\psi\]$ for each intuitionistic Kripke model. 
\end{definition}

\subsection{Undefinability of $\to$ through $\neg$, $\wedge$, and $\vee$}

In this subsection, we use our truth set algebra method to prove that implication $\to$ is not definable in intuitionistic logic through negation $\neg$, conjunction $\wedge$, and disjunction $\vee$. Without loss of generality, assume that language $\Phi_1$ contains only propositional variables $p$ and $q$. Let us consider the Kripke model whose Hasse diagram is depicted in the upper-right corner of Figure~\ref{or figure}. It contains five worlds, $w$, $u$, $v$, $s$, and $t$.  The partial order $\preceq$ on these worlds is given by the diagram. For example, $w\preceq v$ because the diagram contains an upward path from $w$ to $v$.
We assume that $\pi(p)=\{v,s,t\}$ and $\pi(q)=\{v,t\}$.

\begin{figure}
\begin{center}
\vspace{0mm}
\scalebox{0.6}{\includegraphics{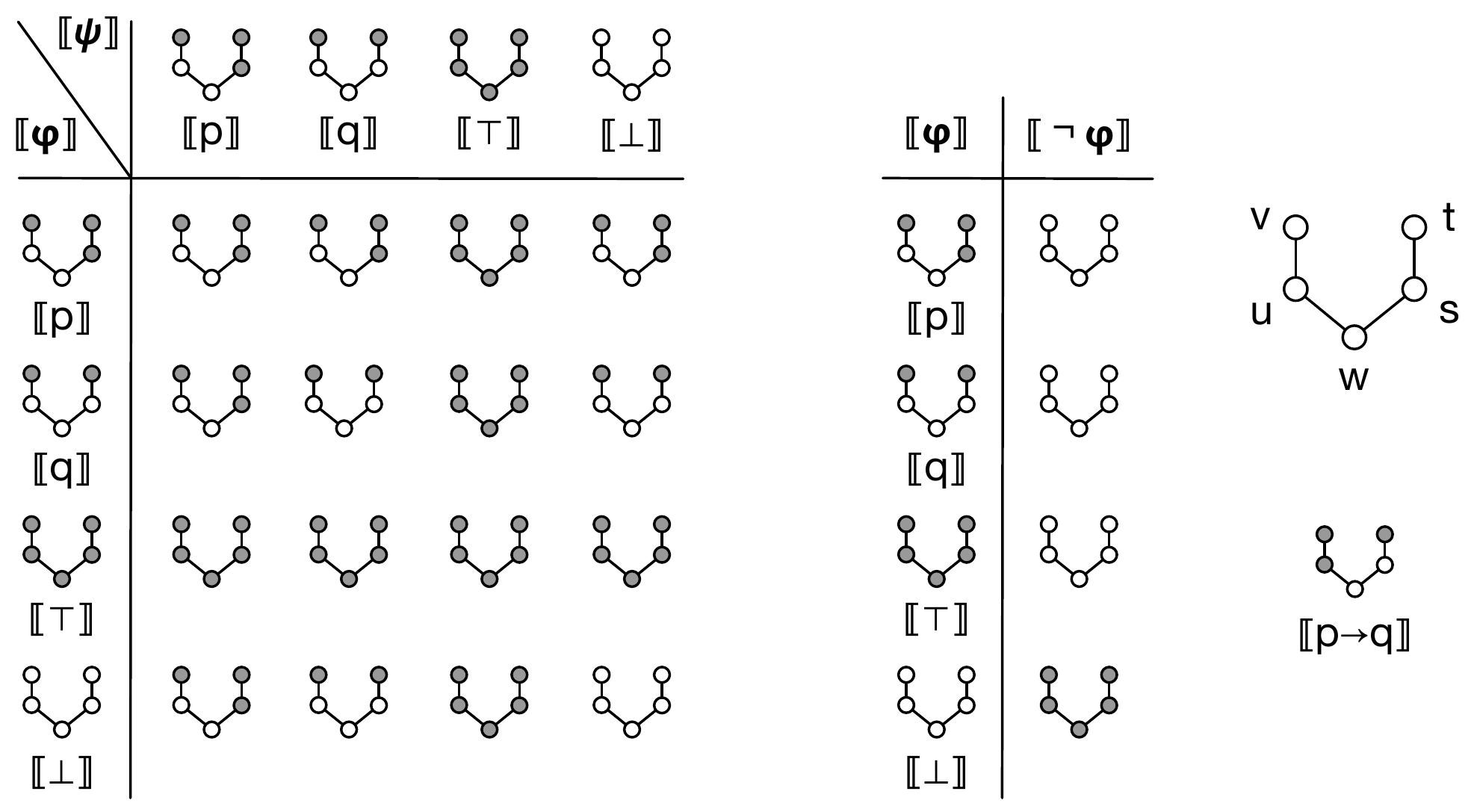}}
\caption{Truth set $\[\phi\vee \psi\]$ for different combinations of truth sets $\[\phi\]$ and $\[\psi\]$ (left). Truth set $\[\neg\phi\]$ for different truth sets $\[\phi\]$ (centre). Hasse diagram for a Kripke model and the truth set $\[p\to q\]$ (right).}
\label{or figure}
\end{center}
\end{figure}

Recall that we define constant $\top$ as $p\to p$ and constant $\bot$ as $\neg\top$. We visualise the truth set of a formula in language $\Phi_1$ by shading the worlds that belong to the set. For example, the rows and the columns in the left-most table in Figure~\ref{or figure} are labelled by the diagrams visualising the truth sets $\[p\]$, $\[q\]$, $\[\top\]$, and $\[\bot\]$. 

\begin{lemma}\label{7-feb-a}
For any formulae $\phi,\psi\in\Phi_1$, if 
$\[\phi\],\[\psi\]\in \{\[p\],\[q\],\[\top\],\[\bot\]\}$, then 
$\[\phi\vee\psi\],\[\phi\wedge\psi\], \[\neg\phi\]\in \{\[p\],\[q\],\[\top\],\[\bot\]\}$.
\end{lemma}
\begin{proof}
Let us first prove that $\[\phi\vee\psi\]\in \{\[p\],\[q\],\[\top\],\[\bot\]\}$ if 
$\[\phi\],\[\psi\]\in \{\[p\],\[q\],\[\top\],\[\bot\]\}$. We do this in the left table depicted in the left of Figure~\ref{or figure}. The proof consists of explicitly constructing the truth set $\[\phi\vee\psi\]$ for each possible combination of sets  $\[\phi\]$ and $\[\psi\]$. 

Alternatively, one can also see, by Definition~\ref{sat intuitionistic} and Definition~\ref{truth set intuitionistic}, that $\[\phi\vee\psi\]=\[\phi\]\cup \[\psi\]$. Then, $\[\phi\vee\psi\]\in \{\[p\],\[q\],\[\top\],\[\bot\]\}$ because the family of truth sets $\{\[p\],\[q\],\[\top\],\[\bot\]\}$ is closed with respect to union.

The proof for the truth set $\[\phi\wedge \psi\]$ is similar: either by building a table or observing that $\[\phi\wedge\psi\]=\[\phi\]\cap \[\psi\]$ and that the family of truth sets $\{\[p\],\[q\],\[\top\],\[\bot\]\}$ is closed with respect to intersection.

Finally, for the truth set $\[\neg\phi\]$, see the middle table in Figure~\ref{or figure}. It shows the truth set $\[\neg\phi\]$ for each formula $\phi$ such that $\[\phi\]\in \{\[p\],\[q\],\[\top\],\[\bot\]\}$. The validity of this table can be verified using item~2 of Definition~\ref{sat intuitionistic}.
\end{proof}

\begin{lemma}\label{9-mar-a}
$\[\phi\]\in \{\[p\],\[q\],\[\top\],\[\bot\]\}$ for any formula $\phi\in\Phi_1$ that does not use implication.
\end{lemma}
\begin{proof}
We prove the statement of the lemma by induction on the structural complexity of formula $\phi$. In the base case, the statement of the lemma is true because the truth sets $\[p\]$ and $\[q\]$ are elements of the family   $\{\[p\],\[q\],\[\top\],\[\bot\]\}$.  

In the induction case, the statement of the lemma follows from Lemma~\ref{7-feb-a} and the induction hypothesis.
\end{proof}

\begin{lemma}
$\[p\to q\]\notin \{\[p\],\[q\],\[\top\],\[\bot\]\}$.    
\end{lemma}
\begin{proof}
We visualise the truth set   $\[p\to q\]$ on the right of Figure~\ref{or figure}.  The validity of this visualisation can be verified using item~5 of Definition~\ref{sat intuitionistic}.
\end{proof}

The next theorem follows from the two lemmas above.

\begin{theorem}[undefinability]
Formula $p\to q$ is not semantically equivalent to any formula in language $\Phi_1$ that does not use implication.    
\end{theorem}

\subsection{Undefinability of $\neg$ through $\wedge$, $\vee$, and $\to$} 

In this subsection, we show that, in intuitionistic logic, negation is not definable through conjunction, disjunction, and implication. Because negation is a unary connective, in this section, without loss of generality, we assume that language $\Phi_1$ contains a single propositional variable $p$.

The proof follows the same pattern as the one in the previous subsection, but it uses a simpler Kripke model. In this case, the Hasse diagram of the model is a tree consisting of a root node and two child nodes: the left child and the right child. Set $\pi(p)$ contains only the left child node. In Figure~\ref{not intuit figure}, we show the truth sets $\[p\]$, $\[\top\]$, and $\[\neg p\]$ for this model. 

\begin{figure}
\begin{center}
\vspace{0mm}
\scalebox{0.53}{\includegraphics{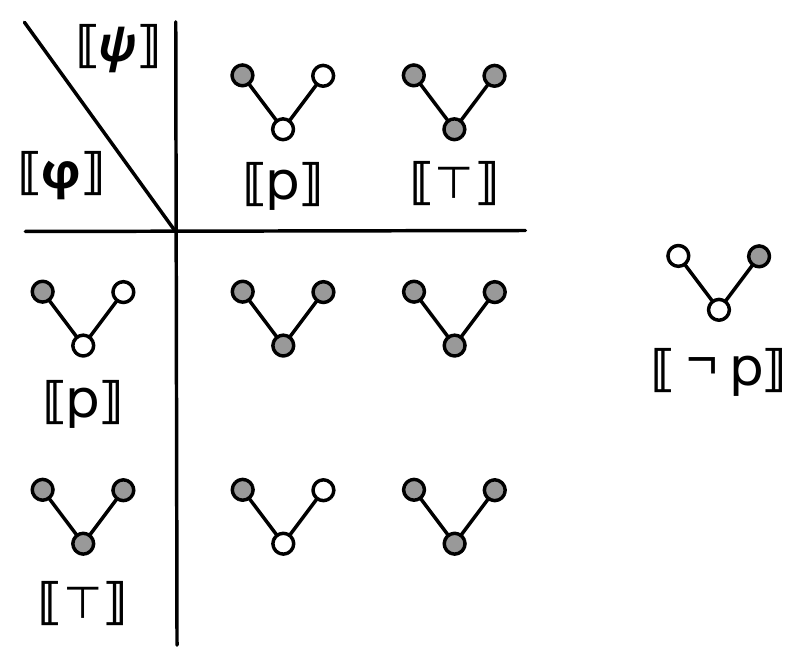}}
\caption{Truth set $\[\phi\to \psi\]$ for different combinations of truth sets $\[\phi\]$ and $\[\psi\]$ (left). Truth set $\[\neg p\]$ (right).}\label{not intuit figure}
\end{center}
\end{figure}

\begin{lemma}\label{9-mar-b}
For any two formulae $\phi,\psi\in\Phi_1$, if 
$\[\phi\],\[\psi\]\in \{\[p\],\[\top\]\}$, then 
$\[\phi\vee\psi\],\[\phi\wedge\psi\], \[\phi\to\psi\]\in \{\[p\],\[\top\]\}$.
\end{lemma}
\begin{proof}
Suppose that $\[\phi\],\[\psi\]\in\{\[p\],\[\top\]\}$. Then,  $\[\phi\]\cup\[\psi\],\[\phi\]\cap\[\psi\] \in\{\[p\],\[\top\]\}$, see visualisation of the truth sets $\[p\]$ and $\[\top\]$ in Figure~\ref{not intuit figure}. Hence, $\[\phi\vee\psi\],\[\phi\wedge\psi\] \in\{\[p\],\[\top\]\}$ by items~3 and 4 of Definition~\ref{sat intuitionistic} and Definition~\ref{truth set intuitionistic}.

On the left of Figure~\ref{not intuit figure}, we visualise the truth set $\[\phi\to\psi\]$ as a function of the truth sets $\[\phi\]$ and $\[\psi\]$. The validity of this table can be verified using item~5 of Definition~\ref{sat intuitionistic} and Definition~\ref{truth set intuitionistic}.
\end{proof}

The proof of the next lemma is similar to the proof of Lemma~\ref{9-mar-a}, but instead of Lemma~\ref{7-feb-a} it uses Lemma~\ref{9-mar-b}.
\begin{lemma}
$\[\phi\]\in \{\[p\],\[\top\]\}$ for any formula $\phi\in\Phi_1$ that does not use negation.
\end{lemma}

\begin{lemma}
$\[\neg p\]\notin \{\[p\],\[\top\]\}$.    
\end{lemma}
\begin{proof}
We visualise the truth set   $\[\neg p\]$ on the right of Figure~\ref{not intuit figure}.  The validity of this visualisation can be verified using item~2 of Definition~\ref{sat intuitionistic}.
\end{proof}

The next theorem follows from the two lemmas above.

\begin{theorem}[undefinability]\label{neg undef intuit}
Formula $\neg p$ is not semantically equivalent to any formula in language $\Phi_1$ that does not use negation.    
\end{theorem}

\subsection{Undefinability of $\vee$ through $\neg$, $\wedge$, and $\to$}

\begin{figure}
\begin{center}
\vspace{0mm}
\scalebox{0.53}{\includegraphics{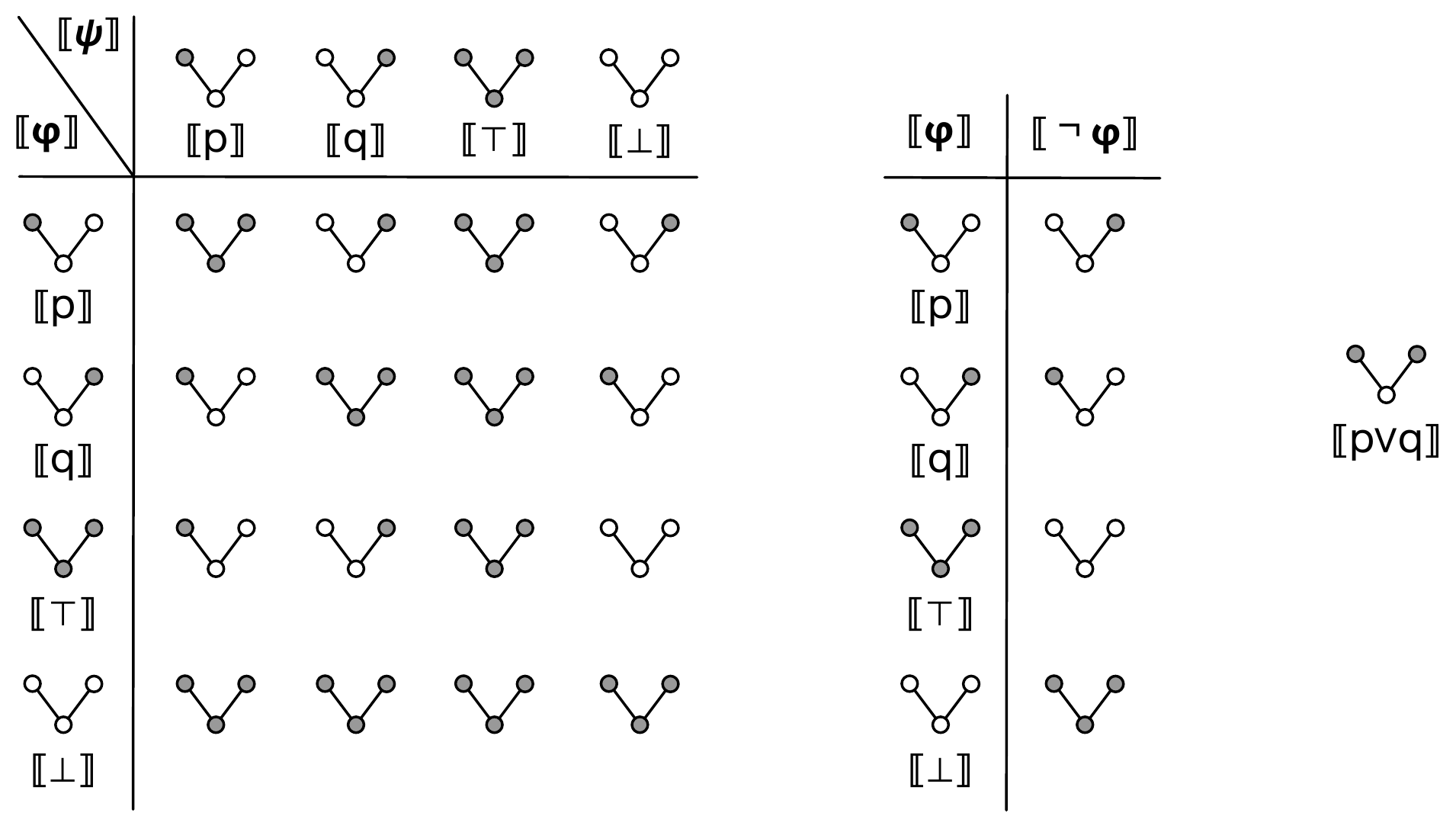}}
\caption{Truth set $\[\phi\to \psi\]$ for different combinations of truth sets $\[\phi\]$ and $\[\psi\]$ (left). Truth set $\[\neg\phi\]$ for different truth sets $\[\phi\]$ (centre). Truth set $\[p\vee q\]$ (right).}\label{or intuit figure}
\end{center}
\end{figure}

The proof of the next theorem is similar to the proof of Theorem~\ref{neg undef intuit} except that it uses Figure~\ref{or intuit figure} instead of Figure~\ref{not intuit figure}.

\begin{theorem}[undefinability]\label{or undef intuit}
Formula $p\vee q$ is not semantically equivalent to any formula in language $\Phi_1$ that does not use disjunction.    
\end{theorem}

\begin{figure}
\begin{center}
\vspace{0mm}
\scalebox{0.53}{\includegraphics{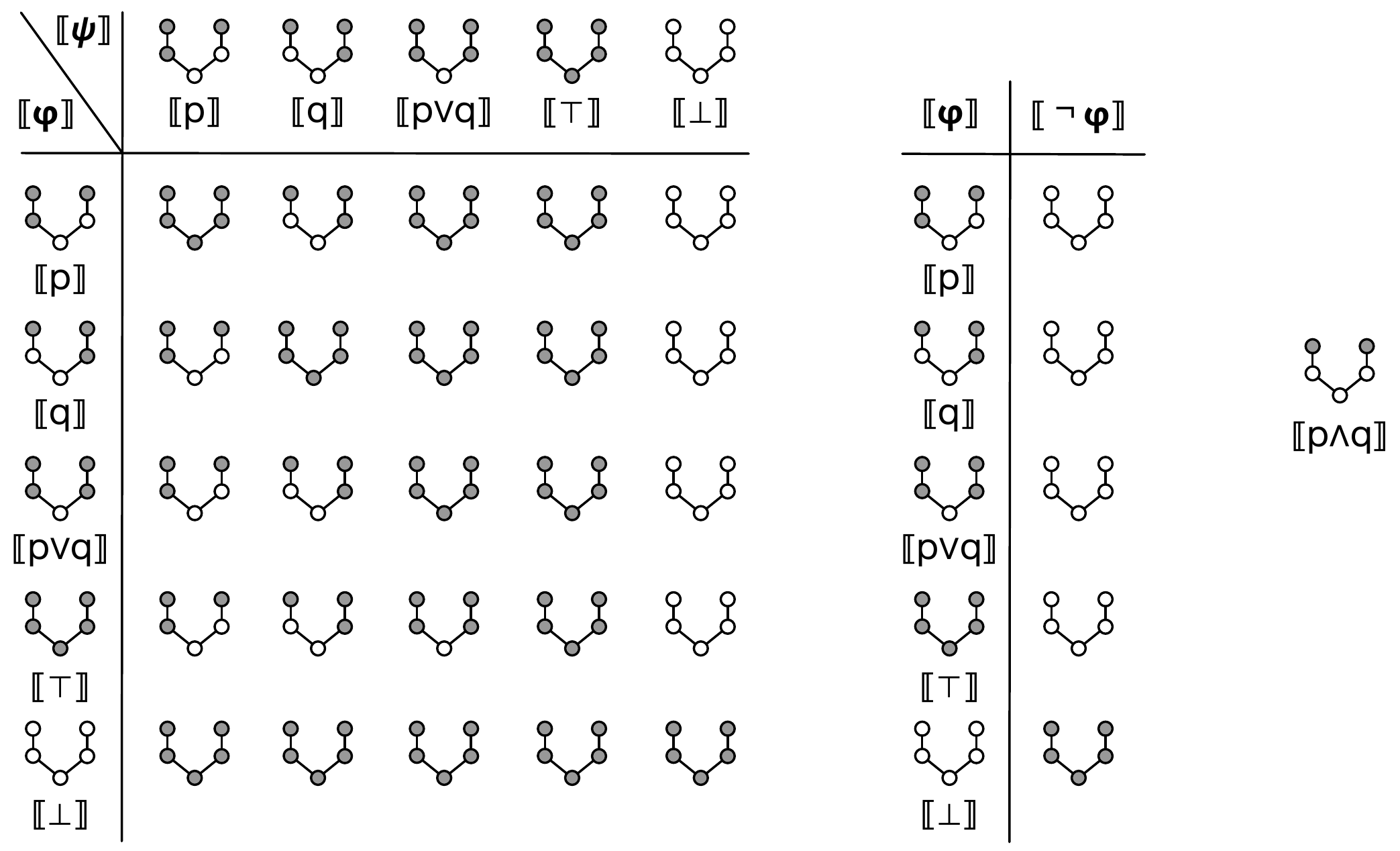}}
\caption{Truth set $\[\phi\to \psi\]$ for different combinations of truth sets $\[\phi\]$ and $\[\psi\]$ (left). Truth set $\[\neg\phi\]$ for different truth sets $\[\phi\]$ (centre). Truth set $\[p\wedge q\]$ (right).}\label{and intuit figure}
\end{center}
\end{figure}

\subsection{Undefinability of $\wedge$ through $\neg$, $\vee$, $\to$}

The proof of the next theorem is similar to the proof of Theorem~\ref{neg undef intuit} except that it uses Figure~\ref{and intuit figure} instead of Figure~\ref{not intuit figure}.


\begin{theorem}[undefinability]\label{and undef intuit}
Formula $p\wedge q$ is not semantically equivalent to any formula in language $\Phi_1$ that does not use conjunction.    
\end{theorem}

\section{Three-Valued Logic}

\begin{figure}[ht]
\begin{center}
\vspace{0mm}
\scalebox{0.6}{\includegraphics{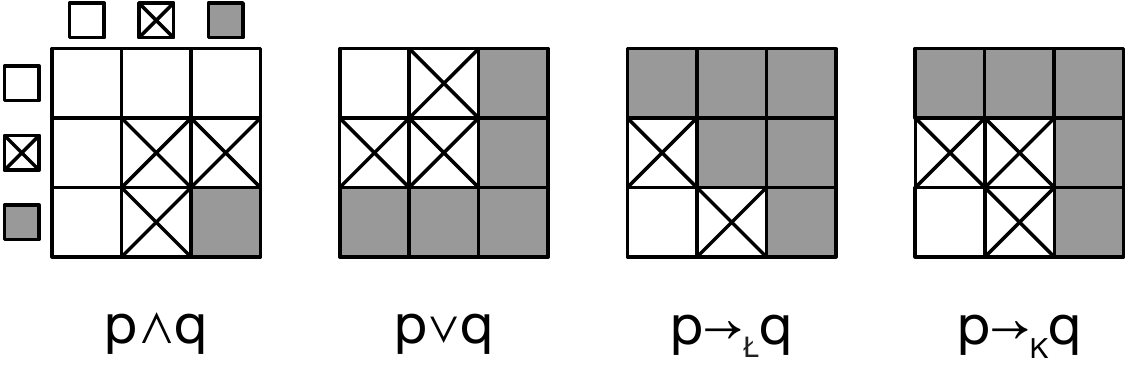}}
\caption{Truth tables for binary connectives in 3-valued logic.}\label{values welcome figure}
\end{center}
\end{figure}
In this section, we apply our technique to investigate the definability of logical connectives in 3-valued logic. This logic contains three truth values: 0, $\frac{1}{2}$, and $1$, often referred to as ``false'', ``unknown'', and ``true'', respectively. The meanings of propositional connectives $\wedge$, $\vee$, and $\neg$ in 3-valued logic are a straightforward generalisation of their meanings in Boolean logic: 
$p\wedge q=\min\{p,q\}$,
$p\vee q=\max\{p,q\}$,
and $\neg p=1- p$.
Thus, for example, if the value of $p$ is ``unknown'', then the value of the expression $p\vee \neg p$ is also ``unknown''. In this article, we visualise values ``false'', ``unknown'', and ``true'' as  a white square, a diagonally crossed square, and a grey square, respectively. The first two diagrams in Figure~\ref{values welcome figure} show truth tables for connectives $\wedge$ and $\vee$. For example, in the left-most diagram, the crossed cell in the middle of the last row represents the fact that if $p=1$ (the third row) and $q=\frac{1}{2}$ (the second column), then $p\wedge q=\frac{1}{2}$.

Defining the meaning of implication in 3-valued logic is less straightforward. Two such definitions are suggested: one by \L ukasiewicz~\cite[p.213]{ll32sl} and the other by Kleene~\cite{k38jsl}. We denote their implications by $\limp$ and $\kimp$, respectively. The truth tables for these implications are shown in the two right-most diagrams in Figure~\ref{values welcome figure}. 
 In this section, we study interdefinability of 3-valued connectives $\neg$, $\wedge$, $\vee$, $\kimp$, and $\limp$. 

 By $\Phi_3$ we denote the language defined by the following grammar:
 $$
\phi:=p\;|\;\neg\phi\;|\;\phi\wedge\phi\;|\;\phi\vee\phi\;|\;\phi\limp\phi\;|\;\phi\kimp\phi,
 $$
 where $p$ is a propositional variable. Because each of the connectives has at most two arguments, for the purposes of proving undefinability, it suffices to assume that there are only two propositional variables, $p$ and $q$.

\subsection{Fuzzy Truth Sets}
 
To apply the truth set algebra technique in the setting of 3-valued logic, we need to make one small modification to this technique. Namely, instead of regular truth sets, we consider {\em fuzzy truth sets} of formulae. In our case, a fuzzy set can have only three degrees of membership: an element can belong, half-belong, or not belong to a fuzzy set. 
 
We consider operations union, intersection, and complement on fuzzy sets. We define the degree of membership in a {\em union} of two fuzzy sets as the {\em maximum} of the degrees of membership in the two original fuzzy sets. For example, suppose fuzzy set $X$ contains an apple and half-contains a banana. In addition, let fuzzy set $Y$ half-contain a banana and contain a carrot. In that case, the union of fuzzy sets $X$ and $Y$ contains an apple, a carrot, and half-contains a banana.

Similarly, we define the degree of membership in an {\em intersection} of two fuzzy sets as the {\em minimum} of the degrees of membership in the two original fuzzy sets. In our example, the intersection of fuzzy sets $X$ and $Y$ half-contains a banana and nothing else. 

Finally, consider any regular (not fuzzy) set $U$ and any fuzzy set $S$ of elements from set $U$. We define a complement of the fuzzy set $S$ with respect to the universe $U$. The degree of the membership of an element in the complement is $1-d$, where $d$ is the degree of membership of the same element in the fuzzy set $S$. In our example, assuming that the universe consists of an apple, a banana, and a carrot, the complement of the fuzzy set $X$ is the fuzzy set $Y$.

Recall our assumption that language $\Phi_3$ contains only propositional variables $p$ and $q$. For any formula $\phi\in\Phi_3$ and any values $b_1,b_2\in \{0,\frac{1}{2},1\}$, by $\phi[b_1,b_2]$ we denote the value of the formula $\phi$ when $p$ has value $b_1$ and $q$ has value $b_2$.
We are now ready to define a fuzzy truth set.
\begin{definition}\label{fuzzy truth set definition}
For any formula $\phi\in\Phi_3$, the fuzzy truth set 
$\[\phi\]$ is a fuzzy set of all pairs $(b_1,b_2)\in \{0,\frac{1}{2},1\}^2$ such that 
\begin{enumerate}
    \item $(b_1,b_2)$ belongs to the fuzzy set $\[\phi\]$ if $\phi[b_1,b_2]=1$,
    \item $(b_1,b_2)$ half-belongs to the fuzzy set $\[\phi\]$ if $\phi[b_1,b_2]=\frac{1}{2}$.
\end{enumerate}   
\end{definition}
We visualise the fuzzy truth set $\[\phi\]$ of an arbitrary formula $\phi$ as a $3\times 3$ table. A cell $(b_1,b_2)$ is coloured white if the pair $(b_1,b_2)$ does not belong to $\[\phi\]$, it is crossed if the pair $(b_1,b_2)$ half-belongs to $\[\phi\]$, and it is coloured grey if the pair $(b_1,b_2)$ belongs to $\[\phi\]$. For example, the four diagrams in Figure~\ref{values welcome figure} visualise the fuzzy truth sets $\[p\wedge q\]$, $\[p\vee q\]$, $\[p\limp q\]$, and $\[p\kimp q\]$.

\begin{definition}\label{25-feb-c}
In the context of 3-valued logic, formulae $\phi,\psi\in\Phi_3$ are semantically equivalent if $\[\phi\]=\[\psi\]$.
\end{definition}

Next, we state and prove a very simple undefinability result about 3-valued logic that does not require the truth set algebra technique. 
\begin{theorem}\label{negation undefinability theorem}
Formula $\neg p$ is not semantically equivalent to any formula containing only connectives $\wedge$, $\vee$, $\kimp$, and $\limp$.
\end{theorem}
\begin{proof}
Observe that if all propositional variables are assigned value $1$, then the value of any formula that contains only connectives $\wedge$, $\vee$, $\kimp$, and $\limp$ is $1$, see Figure~\ref{values welcome figure}. At the same time, the value of $\neg p$ is $0$.   
\end{proof}

\subsection{Expressive Power of Kleene's Implication}

In this subsection, we illustrate how the truth set algebra method can be used to prove undefinability results in 3-valued logic. Namely, we show a relatively simple observation that neither of the other connectives can be defined through Kleene's implication.

\begin{figure}
\begin{center}
\vspace{0mm}
\scalebox{0.6}{\includegraphics{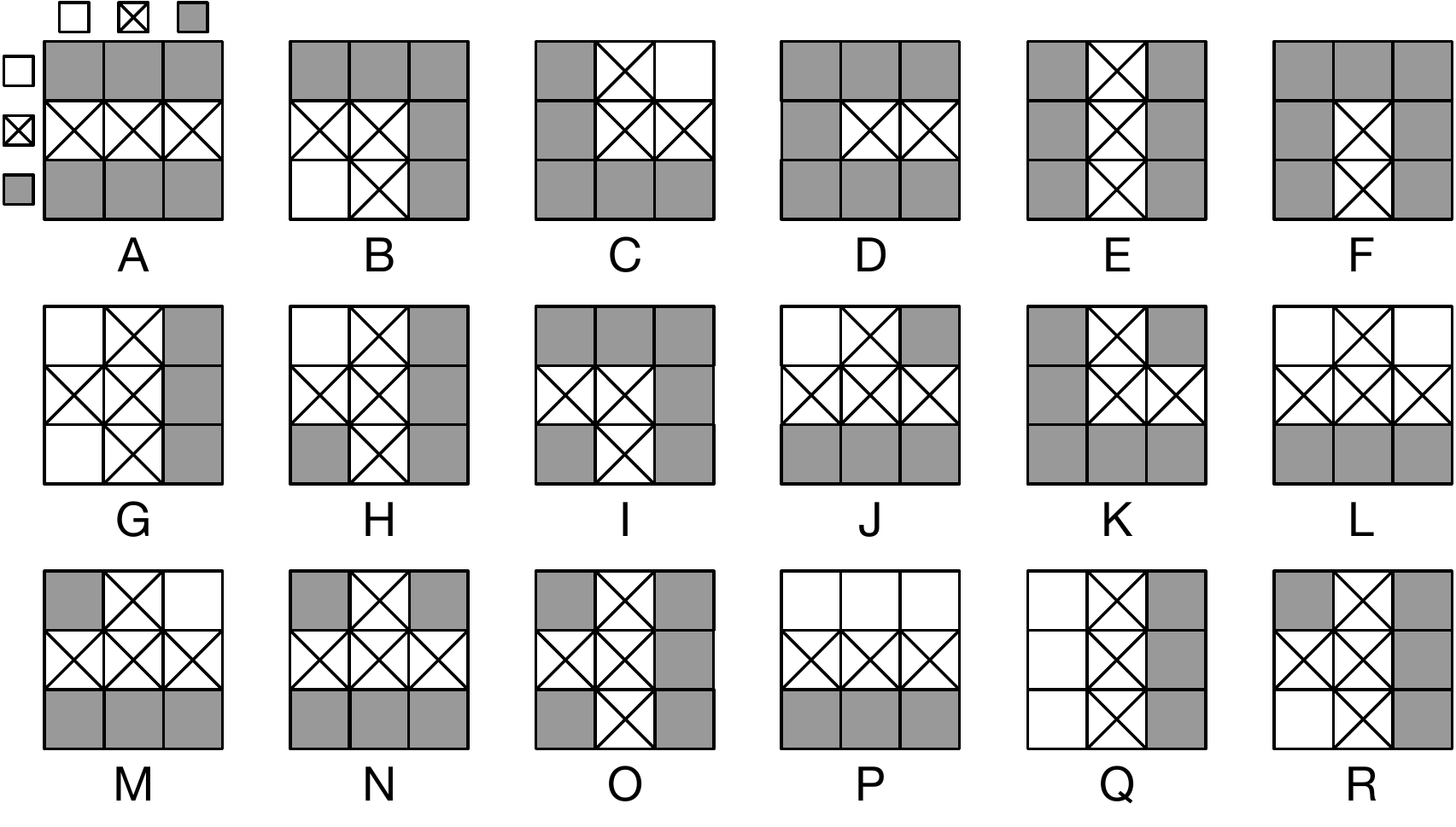}}
\caption{Towards the proof of Theorem~\ref{3-valued via kimp theorem}.}\label{18-diagrams}
\end{center}
\end{figure}
In the rest of this subsection, we use names $A$, \dots, $R$ to refer to the 18 fuzzy truth sets depicted in Figure~\ref{18-diagrams}. Note that $P=\[p\]$ and $Q=\[q\]$. Let $\mathcal{S}$ be the family $\{A,B,C,D,E,F,G,H,I,J,K,L,M,N,O,P,Q,R\}$ of these 18 fuzzy truth sets.

\begin{lemma}\label{25-feb-a}
For any formulae $\phi,\psi\in \Phi_3$, if $\[\phi\],\[\psi\]\in \mathcal{S}$, then $\[\phi\kimp\psi\]\in \mathcal{S}$.
\end{lemma}
\begin{proof}
Consider first the case when $\[\phi\]=A$ and $\[\psi\]=Q$. To compute the fuzzy truth set $\[\phi\kimp\psi\]$, we compute the degree of membership for each pair $(b_1,b_2)$ in this fuzzy set. Consider, for example, the case $b_1=\frac{1}{2}$ and $b_2=0$ which is visualised as the middle-left cell in each diagram. Note that the middle-left cells in the diagrams of fuzzy sets $A$ and $Q$ are crossed and white, respectively, see Figure~\ref{18-diagrams}. Hence, pair $(b_1,b_2)$ half-belongs to the fuzzy truth sets $\[\phi\]=A$ and does not belong to the fuzzy truth set $\[\psi\]=Q$. Thus, by Definition~\ref{fuzzy truth set definition}, the values $\phi[\frac{1}{2},0]$ and $\psi[\frac{1}{2},0]$ are $\frac{1}{2}$ and $0$, respectively. Observe that the value of $\frac{1}{2}\kimp 0$ is $\frac{1}{2}$, see the last diagram in Figure~\ref{values welcome figure}. Hence,  $(\phi\kimp\psi)[\frac{1}{2},0]=\frac{1}{2}$. Then, by Definition~\ref{fuzzy truth set definition}, the pair $(b_1,b_2)$ half-belongs to the fuzzy truth set $\[\phi\kimp\psi\]$. Thus, the middle-left cell in the diagram visualising the fuzzy truth set $\[\phi\kimp\psi\]$ is crossed. By repeating the same computation for each pair $(b_1,b_2)$, one can see that the fuzzy truth set $\[\phi\kimp\psi\]$ is fuzzy set $G$, see Figure~\ref{18-diagrams}. We show this result by placing the letter G in row A, column Q of Table~\ref{3 valued big table}. Therefore, $\[\phi\kimp\psi\]\in\mathcal{S}$.

\begin{table}[]
    \centering
    \begin{tabular}{c|cccccccccccccccccc}
\toprule
 &A&B&C&D&E&F&G&H&I&J&K&L&M&N&O&P&Q&R\\
\midrule
A&A&B&C&D&E&F&G&H&I&J&K&L&M&N&O&P&G&R\\
B&A&I&C&D&E&F&H&H&I&J&K&L&M&N&O&P&H&O\\
C&A&B&K&D&E&F&G&H&I&J&K&J&N&N&O&J&Q&R\\
D&A&B&C&D&E&F&G&H&I&J&K&L&M&N&O&P&Q&R\\
E&A&B&C&D&E&F&G&H&I&J&K&L&M&N&O&L&Q&R\\
F&A&B&C&D&E&F&G&H&I&J&K&L&M&N&O&P&Q&R\\
G&A&I&C&D&E&F&O&O&I&N&K&M&M&N&O&M&O&O\\
H&A&B&C&D&E&F&R&O&I&N&K&M&M&N&O&M&R&R\\
I&A&B&C&D&E&F&G&H&I&J&K&L&M&N&O&P&G&R\\
J&A&B&C&D&E&F&R&O&I&N&K&M&M&N&O&M&R&R\\
K&A&B&C&D&E&F&G&H&I&J&K&L&M&N&O&L&Q&R\\
L&A&B&K&D&E&F&R&O&I&N&K&N&N&N&O&N&R&R\\
M&A&B&K&D&E&F&G&H&I&J&K&J&N&N&O&J&G&R\\
N&A&B&C&D&E&F&G&H&I&J&K&L&M&N&O&L&G&R\\
O&A&B&C&D&E&F&G&H&I&J&K&L&M&N&O&L&G&R\\
P&A&B&D&D&F&F&B&I&I&A&D&A&A&A&I&A&B&B\\
Q&D&F&C&D&E&F&E&E&F&K&K&C&C&K&E&C&E&E\\
R&A&I&C&D&E&F&H&H&I&J&K&L&M&N&O&L&H&O\\
\bottomrule
    \end{tabular}
    \caption{The fuzzy truth set $\[\phi\to_{\text{K}}\psi\]$, where $\[\phi\]$ is the row label and $\[\psi\]$ is the column label.}
    \label{3 valued big table}
\end{table}

The other cases are similar. We show the corresponding fuzzy sets $\[\phi\kimp\psi\]$ in Table~\ref{3 valued big table}. The statement of the lemma holds because all sets in Table~\ref{3 valued big table} belong to family $\mathcal{S}$.
\end{proof}

\begin{lemma}\label{25-feb-b}
$\[\phi\]\in\mathcal{S}$ for any formula $\phi\in\Phi_3$ that uses connective $\kimp$ only.
\end{lemma}
\begin{proof}
We prove the statement of the lemma by induction on the structural complexity of formula $\phi$. If $\phi$ is propositional variable $p$, then $\[p\]=P\in\mathcal{S}$, see Figure~\ref{18-diagrams}. Similarly, if $\phi$ is propositional variable $q$, then $\[q\]=Q\in\mathcal{S}$. If formula $\phi$ has the form $\phi_1\kimp\phi_2$, then the statement of the lemma follows from Lemma~\ref{25-feb-a} and the induction hypothesis.
\end{proof}

\begin{theorem}[undefinability]\label{3-valued via kimp theorem}
Each of the formulae $p\wedge q$, $p \vee q$, and $p\limp q$ is not 3-value-equivalent to a formula that uses connective  
$\kimp$ only.
\end{theorem}
\begin{proof}
The fuzzy truth sets $\[p\wedge q\]$, $\[p\vee q\]$, and $\[p\limp q\]$ are depicted in Figure~\ref{values welcome figure}. Note that none of them belongs to the family $\mathcal{S}$, see Figure~\ref{18-diagrams}. Thus, the statement of the theorem follows from Lemma~\ref{25-feb-b} and Definition~\ref{25-feb-c}.   
\end{proof}

In the rest of this section, we present our main technical results about the connectives $\neg$, $\wedge$, $\vee$, $\kimp$, and $\limp$.

\subsection{Undefinability of Conjunction}

In this subsection, we focus on the definability of conjunction $\wedge$ through the rest of the connectives. First, let us start with three definability facts.  Each of them is easily verifiable using Figure~\ref{values welcome figure} and the definition of negation. In the theorem below and the rest of this section, by $\equiv$ we denote 3-value-equivalence of formulae in language $\Phi_3$.

\begin{theorem}\label{3-value-wedge-definable}
The following equivalences hold in 3-valued logic:
\begin{enumerate}
    \item $p\wedge q \equiv \neg (\neg p \vee \neg q)$,
    \item $p\wedge q \equiv \neg (p \kimp \neg q)$,
    \item $p \wedge q \equiv \neg (p \limp \neg (p\limp q))$.
\end{enumerate}
\end{theorem}
The  first two equivalences in the above theorem are well-known. We are not aware of the third equivalence being mentioned in the literature. It was discovered by our computer program while trying to prove the undefinability of $\wedge$ through $\neg$ and $\kimp$. All three equivalences could be easily verified using the definitions of the connectives.

Let us now discuss the undefinability results about the conjunction. Note that {\em binary} connective $\wedge$ cannot be defined through {\em unary} connective $\neg$. If $\neg$ is combined with any one of the remaining connectives, then $\wedge$ becomes definable, see Theorem~\ref{3-value-wedge-definable}. To completely answer the question about the definability of conjunction, it suffices to show that it cannot be defined without the use of negation. We prove this in the next theorem. 

\begin{theorem}
Formula $p \wedge q$ is not 3-value-equivalent to any formula in language $\Phi_3$ containing only connectives $\vee$, $\kimp$, and $\limp$.    
\end{theorem}
The proof of the above theorem follows the same pattern as the proof of Theorem~\ref{3-valued via kimp theorem}. However, instead of the 18 fuzzy truth sets depicted in Figure~\ref{18-diagrams}, it uses 176 fuzzy truth sets. The equivalent of Table~\ref{3 valued big table} in the new proof is a table containing 176 rows and 176 columns. We used a computer program written in Python to find 176 diagrams like the ones in Figure~\ref{18-diagrams}. The same program also verifies, similarly to how we do in Table~\ref{3 valued big table}, that the set of 176 diagrams is closed with respect to the operations $\vee$, $\kimp$, and $\limp$. Finally, it checks that this set does not contain the diagram for the fuzzy truth set $\[p\wedge q\]$. The algorithm that we used starts with fuzzy truth sets $\[p\]$ and $\[q\]$ and applies the operations $\vee$, $\kimp$, and $\limp$ until no new diagrams could be generated.

It is interesting to point out that the 18 diagrams depicted in Figure~\ref{18-diagrams}, as well as the \LaTeX\ code for Table~\ref{3 valued big table}, are also generated by the same program.

\subsection{Undefinability of Disjunction}

In this subsection, we analyse the definability of disjunction through the rest of the connectives in 3-valued logic. Let us start with the following observation which can be verified using the definitions of the connectives.

\begin{theorem}\label{Disj def results}
The following equivalences hold in 3-valued logic:
\begin{enumerate}
    \item $p\vee q \equiv \neg (\neg p \wedge \neg q)$,
    \item $p\vee q \equiv \neg p \kimp q$,
    \item $p \vee q \equiv (p \limp q) \limp q$.
\end{enumerate}
\end{theorem}
All of the above equivalences are well-known in 3-valued logic. In fact, the last of them is the 3-valued version of Boolean equivalence $\phi\vee\psi\equiv (\phi\to\psi)\to\psi$ that we used in Section~\ref{Classical Propositional Logic section} of this article. Note that Theorem~\ref{Disj def results} shows that the disjunction is definable through $\limp$ alone or also when $\neg$ is used with any other connective. The only case not covered by Theorem~\ref{Disj def results} is resolved in the next theorem.

\begin{theorem}
Formula $p \vee q$ is not 3-value-equivalent to any formula containing only connectives $\wedge$ and $\kimp$.    
\end{theorem}

The computer-generated proof of the above theorem uses 36 fuzzy truth sets.

\subsection{Undefinability of Kleene Implication}

Let us again start with three definability results verifiable through the definitions of the connectives.

\begin{theorem}
The following equivalences hold in 3-valued logic:
\begin{enumerate}
    \item $p\kimp q \equiv \neg p \vee q$,
    \item $p\kimp q \equiv \neg (p \wedge \neg q)$,
    \item $p \kimp q \equiv p \limp \neg(p \limp \neg q)$.
\end{enumerate}
\end{theorem}
The first two equivalences are well-known. The third equivalence was discovered by our computer program. We are not aware of it ever being mentioned in the literature. The only question about the definability of $\kimp$, which is not answered by the above theorem, is answered by the one below.

\begin{theorem}
Formula $p \kimp q$ is not semantically equivalent to any formula containing only connectives $\wedge$, $\vee$, and $\limp$.    
\end{theorem}

The computer proof of this theorem uses 72 diagrams.

\subsection{Undefinability of \L ukasiewicz Implication}

Out of the five connectives that we study only negation (see Theorem~\ref{negation undefinability theorem}) and \L ukasiewicz implication are not definable through the others.

\begin{theorem}\label{limp under theorem}
Formula $p \limp q$ is not semantically equivalent to any formula containing only connectives $\neg$, $\wedge$, $\vee$, and $\kimp$.    
\end{theorem}
The computer proof of the above result uses 82 diagrams. However, in this case, there is a simple argument that does not require the use of a computer. Indeed, if the value of all variables is set to $\frac{1}{2}$ (``unknown''), then the value of any expression that uses only connectives $\neg$, $\wedge$, $\vee$, and $\kimp$ is $\frac{1}{2}$. At the same time the value of $\frac{1}{2}\limp \frac{1}{2}$ is $1$, see Figure~\ref{values welcome figure}. Therefore, connective $\limp$ is not definable through $\neg$, $\wedge$, $\vee$, and $\kimp$.

Although our fuzzy truth sets technique is not required to prove Theorem~\ref{limp under theorem}, this technique could be used to strengthen the theorem. Namely, we can show that connective $\limp$ is not definable through connectives $\neg$, $\wedge$, $\vee$, and $\kimp$ and 3-valued constants  0 (``false''), $\frac{1}{2}$ (``unknown''), and 1 (``true''). The computer proof of this fact already uses 197 diagrams. This is the largest proof mentioned in this section.

Definability results for many other 3-valued connectives are discussed in~\cite{cd13is}. We are not aware of any existing proofs of undefinability in 3-valued logic besides the two proofs mentioned above that don't use fuzzy truth sets: the proof of Theorem~\ref{negation undefinability theorem} and the proof of the original (without constants) version of Theorem~\ref{limp under theorem}.

\section{Conclusion}

In this work, we introduced a new method for proving the undefinability of logical connectives and demonstrated it on examples from Boolean logic, temporal logic, intuitionistic logic, and three-valued logic. Although the technique is potentially applicable to other, more modern logical systems, we have chosen to use these classical examples to make the work accessible to a wider logical audience.


\bibliographystyle{plain}

\bibliography{naumov}

\end{document}